\documentclass[twocolumn,10pt]{article}
\usepackage{scrextend}
\usepackage{geometry}
\usepackage[english]{babel}
\usepackage{amssymb}
\usepackage{tcolorbox}
\usepackage{amsmath}
\usepackage{amsthm}
\usepackage[normalem]{ulem}
\usepackage{float}
\useunder{\uline}{\ul}{}

\usepackage{xcolor}
\usepackage{amsmath}
\usepackage{amsfonts}
\usepackage{bbm}
\usepackage{algorithm}
\usepackage[noend]{algpseudocode}
\usepackage{subcaption}
\usepackage{caption}
\usepackage{subcaption}
\newcounter{mechanism}

\newcommand{\greedy}{GSS}
\newcommand{\dpss}{DPSS}
\newcommand{\newalgo}{DPSS-UCB}

\newtheorem{theorem}{Theorem}
\theoremstyle{definition}
\newtheorem{definition}{Definition}

\newtheorem{lemma}[theorem]{Lemma}

\usepackage{etoolbox}

\preto\equation{\par\nobreak\small\noindent}
\preto\align{\par\nobreak\small\noindent}

\title{A Multi-Arm Bandit Approach To Subset Selection Under Constraints}
\date{}
\author{ Ayush Deva\\\small IIIT Hyderabad, India. \\\small ayushdeva97@gmail.com \and Kumar Abhishek \\\small IIIT Hyderabad, India.\\\small kumar.abhishek@research.iiit.ac.in 
\and Sujit Gujar\\\small IIIT Hyderabad, India.\\\small sujit.gujar@iiit.ac.in}
\usepackage[utf8]{inputenc}
\begin{document}

\maketitle

\begin{abstract}
We explore the class of problems where a central planner needs to select a subset of agents, each with different quality and cost. The planner wants to maximize its utility while ensuring that the average quality of the selected agents is above a certain threshold. When the agents' quality is known, we formulate our problem as an integer linear program (ILP) and propose a deterministic algorithm, namely \dpss\ that provides an exact solution to our ILP.

We then consider the setting when the qualities of the agents are unknown. We model this as a Multi-Arm Bandit (MAB) problem and propose \newalgo\ to learn the qualities over multiple rounds. We show that after a certain number of rounds, $\tau$, \newalgo\ outputs a subset of agents that satisfy the average quality constraint with a high probability. Next, we provide bounds on $\tau$ and prove that after $\tau$ rounds, the algorithm incurs a regret of $O(\ln T)$, where $T$ is the total number of rounds. We further illustrate the efficacy of \newalgo\ through simulations.

To overcome the computational limitations of \dpss, we propose a polynomial-time greedy algorithm, namely \greedy, that provides an approximate solution to our ILP. We also compare the performance of \dpss\ and \greedy\ through experiments. 
\end{abstract}
   
   % We then extend the mechanism to the multi-slot case as \emph{M-LinUCB-MB}. 

%%%%%%%%%%%%%%%%%%%%%%%%%%%%%%%%%%%%%%%%%%%%%%%%%%%%%%%%%%%%%%%%%%%%%%%%

\section{Introduction}
Almost all countries have cooperative societies that cater to developing sectors such as agriculture and handicrafts. We observed that some cooperatives, especially those that are consumer-oriented, such as Coop (Switzerland) or artisan cooperatives who operate their stores, lack a well-defined system to procure products from its many members (manufacturers, artisans, or farmers). Since the production is highly decentralized and usually not standardized, each producer has a different quality and cost of produce depending on various factors such as workmanship and the scale at which it operates. The central planner (say, the cooperative manager) has to carefully trade-off between each producer's qualities and cost to decide the quantity to procure from each producer so that it is most beneficial for the society as a whole.  

This problem is not limited to cooperatives, but it is also faced in other familiar marketplaces. E-commerce platforms, like Amazon and Alibaba, have several sellers registered on their platform. For each product, the platform needs to select a subset of sellers to display on its page while ensuring that it avoids low-quality sellers and does not display only the searched product's high-cost variants. Similarly, a supermarket chain may need to decide the number of apples to procure from the regional apple farmers, each with a different quality of produce, to maximize profits while ensuring that the quality standards are met.

We formulate this as a subset selection problem where a central planner needs to select a subset of these sellers/producers, whom we refer to as agents. In this paper, we associate each agent with its quality and cost of production. The agent's quality refers to the average quality of the units produced by it; however, the quality of an individual unit of its product could be stochastic, especially in artistic and farm products. Thus, it becomes difficult to design an algorithm that guarantees constraint satisfaction on the realized qualities of the individual units procured. Towards this, we show that we achieve probably approximately correct (PAC) results by satisfying our constraint on the expected average quality of the units procured. Every unit procured from these agents generates revenue that is a function of its quality. The planner aims to maximize its utility (i.e., revenue - cost) while ensuring that the procured units' average quality is above a certain threshold to guarantee customer satisfaction and retention \cite{achilleas2008marketing,terziovski1997business}. When the agents' quality is known, we model our problem as an Integer Linear Program (ILP) and propose a novel algorithm, \dpss\ that provides an exact solution to our ILP.

Often, the quality of the agents is unknown to the planner beforehand. An E-commerce platform may not know its sellers' quality at the time of registration, and an artisan's quality of work may be hard to estimate until its products are procured and sold in the market. Thus, the planner needs to carefully learn the qualities by procuring units from the agents across multiple rounds while minimizing its utility loss.
Towards this, we model our setting as a Multi-Arm Bandit (MAB) problem, where each agent represents an independent arm with an unknown parameter (here, quality). To model our subset selection problem, we consider the variant of the classical MAB setting where we may select more than one agent in a single round. This setting is popularly referred to as a Combinatorial MAB (CMAB) problem \cite{combes15,Gai12,kveton15}.  In studying CMAB, we consider the semi-bandit feedback model where the algorithm observes the quality realizations corresponding to each of the selected arms and the overall utility for selecting the subset of arms. The problem becomes more interesting when we also need to ensure our quality constraint in a CMAB problem. We position our work with respect to the existing literature in Section \ref{sec:rltdWork}.

 Typically, in a CMAB problem, the planner's goal is to minimize the \textit{expected regret}, i.e., the difference between the expected cumulative utility of the best offline algorithm with known distributions of an agent's quality and the expected cumulative reward of the algorithm. However, the traditional definition of regret is not suitable in our setting as an optimal subset of agents (in terms of utility) may violate the quality constraint. Thus, we modify the regret definition to make it compatible with our setting. We propose a novel, UCB-inspired algorithm, \newalgo, that addresses the subset selection problem when the agents' quality is unknown. We show that after a certain threshold number of rounds, $\tau$, the algorithm satisfies the quality constraint with a high probability for every subsequent round, and under the revised regret definition, it incurs a regret of $O(\ln T)$, where $T$ is the total number of rounds. 

To address the computational challenges of \dpss\, which has a time complexity of $O(2^n)$, we propose a greedy-based algorithm, \greedy\ that runs in polynomial time $O(n \ln n)$, where $n$ is the number of agents. We show that while the approximation ratio of the utility achieved by \greedy\ to that of \dpss\ can be arbitrarily small in the worst case, it achieves almost the same utility as \dpss\ in practice, which makes \greedy\ a practical alternative to \dpss\, especially when $n$ is large.

In summary, our contributions are:
\begin{itemize}
    \item We propose a framework, SS-UCB, to model subset selection problem under constraints when the properties (here, qualities) of the agents are unknown to the central planner. In our setting, both the objective function and the constraint depends on the unknown parameter. 
    %We also provide a regret definition to compare the performance of an algorithm for such a setting.
    \item We first formulate our problem as an ILP assuming the agents' quality to be known and propose a novel, deterministic algorithm, namely \dpss (Algorithm \ref{algo:dpss}) to solve the ILP.
    \item Using \dpss, we design \newalgo\, which addresses the setting where the agents' quality is unknown. We prove that after a certain number of rounds, $\tau = O(\ln T)$, \newalgo\ satisfies quality constraint with high probability. We also prove that it achieves a regret of $O(\ln T)$ (Theorem \ref{thm:tau}).
    %under our regret definition. 
    \item To address the computational limitation of \dpss, we propose an alternative greedy approach, \greedy\, and GSS-UCB, that solves the known and the unknown settings, respectively. We show that while the greedy approach may not be optimal, it performs well in practice with a huge computational gain that allows our framework to scale to settings with a large number of agents.
\end{itemize}

The remaining of the paper is organized as follows: In Section \ref{sec:rltdWork}, we discuss the related works. In Section \ref{sec:prFo}, we define our model and solve for the setting when the quality of the agents are known. In Section \ref{sec:unknown}, we address the problem when the quality of the agents is unknown. In Section \ref{sec:greedyApp}, we propose a greedy approach to our problem. In Section \ref{sec:exp}, we discuss our simulation-based analysis and conclude the paper in Section \ref{sec:concl}.

\section{Related Work}
\label{sec:rltdWork}
Subset selection is a well-studied class of problems that finds its applications in many fields, for example, in retail, vehicle routing, and network theory. Usually, these problems are modeled as knapsack problems where a central planner needs to select a subset of agents that maximizes its utility under budgetary constraints \cite{zaimai1989optimality}. There are several variations to the knapsack, such as robustness \cite{rooderkerk2016robust}, dynamic knapsacks \cite{papastavrou1996dynamic}, and knapsack with multiple constraints \cite{sinha1979multiple} studied in the literature. In this paper, we consider a variant where the constraint is not additive, i.e., adding another agent to a subset doesn't always increase the average quality.

When online learning is involved, the stochastic multi-armed bandit (MAB) problem captures the exploration vs. exploitation trade-off effectively \cite{ho13,karger2011,Jain2016,jain18,slivkins2019,Tran13,TRANTHANH2014,Biswas15}. The classical MAB problem involves learning the optimal agent from a set of agents with a fixed but unknown reward distribution \cite{bubeck2012regret,slivkins2019, auer2002, thompson1933likelihood}. Combinatorial MAB (CMAB) \cite{Chen13,gai10,shou_NIPS2014,combes_NIPS2015,chenNIPS16} is an extension to the classical MAB problem where multiple agents can be selected in any round. In \cite{Chen13,Gai12,chen2016combinatorial}, the authors have considered a CMAB setting where they assume the availability of a feasible set of subsets to select from. The key difference with our setting is that our constraint itself depends on the unknown parameter (quality) that we are learning through MAB. Thus, the feasible subsets that satisfy the constraint need to be learned, unlike the previous works. \cite{Chen13,Gai12,chen2016combinatorial} also assumes the availability of an oracle that outputs an optimal subset given the estimates of the parameter as input, whereas we design such an oracle for our problem. Bandits with Knapsacks (BwK) is another interesting extension that introduces constraints in the standard bandit setting \cite{badanidiyuru13,badanidiyuru14,agrawal14,jain18} and finds its applications in dynamic pricing, crowdsourcing, etc. (see \cite{badanidiyuru13,agrawal14}). Typically, in BwK, the objective is to learn the optimal agent(s) under a budgetary constraint (e.g., a limited number of selections) that depends solely on the agents' cost. However, we consider a setting where the selected subset needs to satisfy a quality constraint that depends on the learned quantities.

The closest work to ours is \cite{jain18} where the authors present an assured accuracy bandit (AAB) framework where the objective is to minimize cost while ensuring a target accuracy level in each round. While they do consider a constraint setting similar to ours, the objective function in \cite{jain18} depends only on the agents' cost and not on the qualities of the agents that are unknown. Hence, it makes our setting different and more generalizable with respect to both AAB and CMAB as in our setting, both the constraint and the utility function depend on the unknown parameter.

\section{Subset Selection With Known Qualities of Agents}
\label{sec:prFo}
Here we assume that the agents' quality is known and consider the problem where a central planner $ C $ needs to procure multiple units of a particular product from a fixed set of agents. Each agent is associated with the quality and cost of production. $ C $'s objective is to procure the units from the agents such that the average quality of all the units procured meets a certain threshold. We assume that there is no upper limit to the number of units it can procure as long as the quality threshold is met. 

In Section \ref{sec:model}, we define the notations required to describe our model, formulate it as an integer linear program (ILP) in Section \ref{sec:ilp}, and propose a solution to it in Section \ref{sec:dpss}.

\subsection{Model and Notations}
\label{sec:model}
\begin{enumerate}
    \item There is a fixed set of agents $N$ = $\{1,2,\ldots, n\}$ available for selection for procurement by planner $C$.  %who can supply a particular product of interest. 
    \item Agent $i$ has a cost of production, $c_i$, and capacity, $k_i$ (maximum number of units it can produce).
    \item The quality of the $j^{th}$ unit of produce by agent $i$ is denoted by $Q_{ij}$, which we model as a Bernoulli random variable.
    \item For any agent $i$, the probability that $Q_{ij}$ is $1$ is defined by $q_i$, i.e., $E[Q_{ij}] = q_i$ for any unit $j$ procured from agent $i$. $q_i$ is also referred to as the quality of the agent in the rest of the paper.
    \item The utility for $C$ to procure a single unit of produce from agent $i$ is denoted by $r_i$, which is equal to its expected revenue\footnote{We assume expected revenue to be proportional to the quality of the product. It is a reasonable assumption as if $q_i$ is the probability of the product being sold and $R$ is the price of the product, its expected revenue would be $Rq_i$} minus the cost of production, i.e., $r_i = Rq_i - c_i$, where $R$ is the proportionality constant.
    \item The quantity of products procured by $C$ from the $i^{th}$ agent is given by $x_i$. 
    \item The average quality of products procured by $C$ is therefore equal to $\frac{\sum_{i\in N}\sum_{j=1}^{x_i}Q_{ij}}{\sum_{i\in N}x_i}$. \label{itm:avg_quality}
    \item We define $q_{av} = \frac{\sum_{i \in N}x_iq_i}{\sum_{i \in N}x_i}$, which is the expected average quality of the units procured by $C$.
    \item $C$ needs to ensure that the average quality of all the units procured is above a certain threshold, $\alpha \in [0,1]$.
    \item The total  utility of $C$ is given by, $z = \sum_{i \in N}x_ir_i$.
\end{enumerate}
Usually, an individual unit's quality $Q_{ij}$ may not be quantifiable and can only be characterized by observing whether it was sold. Hence, we model it as a Bernoulli random variable.

\subsection{Ensuring Quality Constraints}
\label{sec:realized}

In our setting, average quality (Section \ref{sec:model}, point \ref{itm:avg_quality}) is dependent on $Q_{ij}$, which is stochastic in nature. In such a stochastic framework, it is more natural to work with expected terms than on a sequence of realized values. Towards this, we show that by ensuring our quality constraint on expected average quality, $q_{av}$, instead, we can still achieve approximate constraint satisfaction with a high probability. Formally, we present the following lemma, 
\begin{lemma}
\label{lemma:realized}
The probability that average quality is less than $\alpha - \epsilon$ given that $q_{av} \geq \alpha$, can be bounded as follows:
  \begin{displaymath}
    \mathcal{P}\left(\ \frac{\sum_{i\in N}\sum_{j=1}^{x_i}Q_{ij}}{\sum_{i\in N}x_i} < \alpha - \epsilon\ \mid  q_{av} \geq \alpha \right) \leq \exp{(-2\epsilon^2m))},
  \end{displaymath} where $m = \sum_{i\in N}x_i$, and $\epsilon$ is a constant.
\end{lemma}

\begin{proof}
Let, $V = \frac{\sum_{i\in N}\sum_{j=1}^{x_i}Q_{ij}}{\sum_{i\in N}x_i}$
\begin{align*}
E[V] &= \frac{\sum_{i\in N}\sum_{j=1}^{x_i}E[Q_{ij}]}{\sum_{i\in N}x_i} = \frac{\sum_{i \in N} q_ix_i}{\sum_{i\in N}x_i} = q_{av}
\end{align*}
Therefore,
\begin{align*}
\mathcal{P}\big(V < \alpha - \epsilon\ | \ E[V] \geq \alpha \big)
 \leq \mathcal{P}\big(V < E[V] - \epsilon)\\
 = \mathcal{P}\big(V - E[V] < - \epsilon) &\leq exp(-2\epsilon^2m)
\end{align*}
The last line follows from the Hoeffding's inequality \cite{hoeffding1963}.
\end{proof}
From the above lemma, we show that by ensuring $q_{av} \geq \alpha$, we can achieve probably approximate correct (PAC) results on our constraint. Hence, for the rest of the paper, we work with $q_{av} \geq \alpha$ as our quality constraint (QC).

\subsection{Integer Linear Program (ILP)}
\label{sec:ilp}
When the qualities of the agents are known, the planner's subset selection problem can be formulated as an ILP where it needs to decide on the number of units, $x_i$, to procure from each agent $i$ so as to maximize its utility (objective function) while ensuring the quality and capacity constraints. The optimization problem can be described as follows:
\begin{center}
\fbox{\begin{minipage}{0.8\columnwidth}
\begin{equation}
\begin{aligned}
%\hline
\label{eq:opt_prb}
\max_{x_i} \quad & \sum_{i\in N}(Rq_i-c_i)x_i\\
\textrm{s.t. } & q_{av} = \frac{\sum_{i\in N}q_ix_i}{\sum_{i\in N}x_i} \\
    & q_{av} \geq \alpha    \\
    \quad & 0 \leq x_i \leq k_i & \forall i \in N \\ 
    & x_i \in \mathbb{Z} &
    \forall i \in N \\ %\hline 
    \end{aligned}
\end{equation}
\end{minipage}}
\end{center}

%%%%%%%%%%%%%%%%%%%%%%%%%%%%%%%%%%%%%%%%%%%%%%%%%%%%%%%%%%%%%%%%%%%%%%%%

\subsection{Dynamic Programming Based Subset Selection (DPSS)}
\label{sec:dpss}
In order to solve the ILP, we propose a dynamic programming based algorithm, called \dpss. For ease of exposition, we consider $k_i=1$, i.e., each agent has a unit capacity of production. This is a reasonable assumption that doesn't change our algorithm's results, since, for an agent with $k_i > 1$, we can consider each unit as a separate agent, and the proofs and discussion henceforth follows.  

Formally, the algorithm proceeds as follows:
\begin{enumerate}
    \item Divide the agents into one of the four categories:
    \begin{enumerate}
        \item \textbf{$S_1$}: Agents with $q_i \geq \alpha$ and $r_i \geq 0$
        \item \textbf{$S_2$}: Agents with $q_i < \alpha$ and $r_i \geq 0$
        \item \textbf{$S_3$}: Agents with $q_i \geq \alpha$ and $r_i < 0$
        \item \textbf{$S_4$}: Agents with $q_i < \alpha$ and $r_i < 0$
    \end{enumerate}
    \item Let $\boldsymbol{x} = \{x_i\}_{i\in N}$ be the selection vector, where $x_i = 1$ if the $i^{th}$ agent is selected and 0 otherwise.
    \item Since an agent in $S_1$ has a positive utility and above threshold quality,  $x_i = 1$, $\forall i \in S_1$. Let $d = \sum_{i \in S_1}(q_i - \alpha$) be the excess quality accumulated. 
    \item Similarly, all units in $S_4$ have a negative utility and below threshold quality. Hence, $x_i = 0$, $\forall i \in S_4$. 
    \item Let $G$ be the set of the remaining agents (in $S_2$ and $S_3$). For each agent $i \in G$, we define $d_i = q_i - \alpha$. Thus, we need to select the agents $i \in G$ that maximizes the utility, such that $\sum_{i\in G} x_i d_i \leq d $.
    \item For agents in G, select according to the DP function defined in Algorithm \ref{algo:dpss} (Lines [\ref{algl:dpStrt}-\ref{algl:dpEnd}]). Here, $d^{te}$ denotes the access quality accumulated before choosing the next agent and $x^te$ refers to the selections made so far in the DP formulation
\end{enumerate}{}
\begin{algorithm}[!ht]
\caption{DPSS}
\label{algo:dpss}
%\algsetup{indent=0.75em}
\begin{small}
%\SetAlgoLined
% \KwResult{Write here the result }
\begin{algorithmic}[1]
\State \textbf{Inputs:} $N$, $\alpha$, $R$, costs  $\boldsymbol{c} = \{c_i\}_{i \in N}$, qualities $\boldsymbol{q} = \{q_i\}_{i \in N}$
\State \textbf{Output:} Quantities procured $\boldsymbol{x} = (x_1,\ldots,x_n)$
\State \textbf{Initialization:} $\forall i \in N$, $r_i = Rq_i - c_i$, $z=0$
\State Segregate $S_1$,$S_2$,$S_3$,$S_4$ as described in Section 4.1
\State $\forall i \in S_1$, $x_i = 1$; $z = z + r_i$; $d = \sum_{i \in S_1}(q_i - \alpha)$ 
\State $\forall i \in S_4$, $x_i = 0$
\State $G = S_2 \cup S_3$ ; $\forall i \in G, d_i = q_i - \alpha$
\Function {dp} {$i, d^{te}, x^{te}, x^{\star}, z^{te}, z^{\star}$} \label{algl:dpStrt}
\If{$i == |G|$ and $d^{te}< 0$} \Return $x^{\star}, z^{\star}$
\EndIf
\If{$i == |G|$ and $d^{te} \geq 0$}
\If{$z^{te} > z^{\star}$}
\State $z^{\star} = z^{te}$; $x^{\star} = x^{te}$
\EndIf
\State  \textbf{return} $x^{\star}, z^{\star}$
\EndIf
\State $x^{\star}, z^{\star} = DP(i+1, d^{te}, [x^{te},0], x^{\star}, z^{te}, z^{\star})$
\State $x^{\star}, z^{\star} = DP(i+1, d^{te} + d_{i}, [x^{te},1], x^{\star}, z^{te}+ r_i, z^{\star})$
\State  \textbf{return} $x^{\star}, z^{\star}$ \label{algl:dpEnd}
\EndFunction
\State $x^G, z^G$ = DP(0,d,[\ ],[\ ],0,0)
\State $\forall i \in G, x_i = x^G_i$
\State \textbf{return} x
%\State \textbf{Function} $dp(i, d, x^{te}, x^{\star}, z^{te}, z^{\star}$):\\
\end{algorithmic}
\end{small}
\end{algorithm}

%%%%%%%%%%%%%%%%%%%%%%%%%%%%%%%%%%%%%%%%%%%%%%%%%%%%%%%%%%%%%%%%%%%%%%%%

\section{Subset Selection with Unknown Qualities of Agents}
\label{sec:unknown}
In the previous section, we assumed that the qualities of the agents, $q_i$, are known to \emph{C}. We now consider a setting when $q_i$ are \emph{unknown} beforehand and can only be learned by selecting the agents. We model it as a CMAB problem with semi-bandit feedback and QC.

\subsection{Additional Notations}
\label{sec:addNot}
We introduce the additional notations to model our problem. Similar to our previous setting, we assume that we are given a fixed set of agents, $N$, each with its own average quality of produce, $q_i$ and cost of produce, $c_i$. Additionally, our algorithm proceeds in discrete rounds $t = 1,\ldots,T$. For a round $t$:

\begin{itemize}
    \item Let $\boldsymbol{x}^t \in \{0,1\}^n$ be the selection vector at round $t$, where $x_i^t = 1$ if the agent $i$ is selected in round $t$ and $x_i^t = 0$ if not.
    
    \item The algorithm selects a subset of agents, $S^t \subseteq N$, referred to as a super-arm henceforth, where $S^t = \{i \in N | x_i^t = 1\}$. Let $s^t$ be cardinality of selected super-arm, i.e., $s^t = |S^t|$.  
    \item Let $w_i^t$ denote the number of rounds an agent $i$ has been selected until round $t$, i.e., $w_i^t = \sum_{y\leq t} x_i^y$.
    \item For each agent $i \in S^t$, the planner, $C$, observes its realized quality $X_i^j$, where $j = w_i^t$ and E[$X_i^j$] = $q_i$. For an agent $i \notin S_t$, we do not observe its realized quality (semi-bandit setting). 
    \item  The empirical mean estimate of $q_i$ at round $t$, is denoted by $\hat{q}_i^t = \frac{1}{w_i^t}\sum_{j=1}^{w_i^t}X_i^j$. The upper confidence bound (UCB) estimate is denoted by $(\hat{q}_i^t)^+ = \hat{q}_i^t+\sqrt{\frac{3\ln t}{2w_i^t}} $.
    \item Utility to $C$ at round $t$ is given by:
    $r_{\boldsymbol{q}}(S^t) = \sum_{i \in S^t}Rq_i-c_i$, where $\boldsymbol{q} = \{q_1,q_2,\ldots,q_n\}$ is the quality vector.
    \item The expected average quality of selected super-arm at round $t$ is given by: $q_{av}^t$ = $\frac{1}{s^t}\sum_{i \in S^t}q_i$. 
\end{itemize}
Following from Lemma \ref{lemma:realized}, we continue to work with expected average quality instead of realized average quality.
\subsection{SS-UCB}
\label{ssec:ssucb}
In this section, we propose an abstract framework, SS-UCB, for subset selection problem with quality constraint. SS-UCB assumes that there exist an offline subset selection algorithm, SSA, (e.g., DPSS), which takes a vector of qualities, $\boldsymbol{q'}$, and costs, $\boldsymbol{c'}$, along with the target quality threshold, $\alpha'$, and proportionality constant, $R$, as an input and returns a super-arm which satisfies the quality constraint (QC) with respect to $\boldsymbol{q'}$ and $\alpha'$.

SS-UCB runs in two phases: (i) Exploration: where all the agents are explored for certain threshold number of rounds, $\tau$ ; (ii) Explore-exploit: We invoke SSA (line \ref{ss-ucb:Oracle}, Algorithm \ref{algo:ss_ucb}) with $\{(\hat{q}_i^t)^+\}_{i \in N}$, $\{c_i\}_{i \in N}$, $\alpha + \epsilon_2$ and $R$ as the input parameters and select accordingly. We invoke SSA with a slightly higher target threshold, $\alpha + \epsilon_2$, so that our algorithm is more conservative while selecting the super-arm in order to ensure QC with a high probability (discussed in Section \ref{sec:correct}). As we shall see in Section \ref{sec:correct}, the higher the value of $\epsilon_2$, the sooner the SSA satisfies QC with a high probability but it comes with the cost of loss in utility. Thus, the value of $\epsilon_2$ must be appropriately selected based on the planner's preferences.

We refer to the algorithm as DPSS-UCB when we use \dpss\ (Algorithm \ref{algo:dpss}) as SSA in the SS-UCB framework. We show that DPSS-UCB outputs the super-arm that satisfies the QC 
% (with respect to the true qualities of the agents, $\boldsymbol{q} = \{q_i\}_{i\in N}$) 
with high probability (w.h.p) after a certain threshold number of rounds, $\tau$, and incurs a regret of $O(\ln T)$.

\if 0
\begin{algorithm}
\SetAlgoLined
\caption{GSS - UCB}
\label{algo:greedy}
% \KwResult{Write here the result }
\SetKwInOut{Input}{Input}
\SetKwInOut{Output}{Output}
\Input{Set of producers $N$, $\alpha$, $\epsilon_2$, $R$\\
costs $ c = (c_1,c_2,\ldots,c_n)$, \\capacities $k = (k_1,\ldots,k_n)=(1,\ldots,1)$ }
\Output{At each round $t$, Quantities procured $\mathbf{x}^t = (x_1,x_2,\ldots,x_n)$ and revenue $z^t$}
For each agent, maintain: $w_i^t$ as the number of times agent, $i$ has been chosen so far; estimated mean,$\tilde{q_i}$ as the average of all the realizations of the quality of products from that agent so far. \\
$\tau \leftarrow $ Calculate\_tau
($\epsilon_2$)\\
$t$ = 0

\While{$t\leq \tau$  (\textbf{Explore Phase})}{
Play a Super-arm $S=N$, update $w_i^t$, $\hat{q}_i^t$ and $(\hat{q}_i^t)^+$.\\
$t\leftarrow t+1$
}

\While{$t\leq T$ (\textbf{Explore-Exploit Phase})}{
For each agent $i$, set $(\hat{q}_i^t)^+ = \hat{q}_i^t + \sqrt{\frac{3 \ln t}{2w_i^t}}$\\
$S$ = GSS($[(\hat{q}_i^t)^+],c,k,\alpha+\epsilon_2,R$)\\
Play $S$ and update $w_i^t$ and $\hat{q}_i^t$ \ \ $\forall i \in N$\\
$t\leftarrow t+1$}
\end{algorithm}
\fi 

\begin{algorithm}[!ht]
\caption{SS-UCB}
\label{algo:ss_ucb}
\begin{small}
%\SetAlgoLined
% \KwResult{Write here the result }
\begin{algorithmic}[1]
\State \textbf{Inputs:} $N$, $\alpha$, $\epsilon_2$, R,
costs  $\boldsymbol{c} = \{c_i\}_{i \in N}$ 
% \State \textbf{Output:} At each round $t$, Selection vector $\boldsymbol{x}^t = (x_1^t,x_2^t,\ldots,x_n^t)$ and revenue $z^t$
\State For each agent $i$, maintain: $w_i^t$, ${q_{i}^t}$, $(\hat{q}_i^t)^+$
\State $\tau \leftarrow \frac{3\ln T}{2 \epsilon_2^2}$; $t$ = 0

\While{$t\leq \tau$  (\textbf{Explore Phase})}
\State Play a super-arm $S^t=N$ 
\State Observe qualities $X_{i}^j, \forall i\in S^t$ and update $w_i^t$, $\hat{q}_i^t$
\State $t\leftarrow t+1$
\EndWhile

\While{$t\leq T$ (\textbf{Explore-Exploit Phase})}
\State For each agent $i$, set $(\hat{q}_i^t)^+ = \hat{q}_i^t + \sqrt{\frac{3 \ln t}{2w_i^t}}$
\State $S^t$ = SSA ($\{(\hat{q}_i^t)^+\}_{i \in N},c,\alpha+\epsilon_2$,R) \label{ss-ucb:Oracle}
\State Observe qualities $X_{i}^j, \forall i\in S^t$ and update $w_i^t$, $\hat{q}_i^t$
\State $t\leftarrow t+1$
\EndWhile
\end{algorithmic}
 \end{small}
\end{algorithm}

%Note that, as against a typical MAB algorithm, we select all the agents for first $\tau$ rounds. 
% In \cite{jain18}, the authors have proved after a threshold number of rounds, the algorithm ensures QC is satisfied with high probability. Motivated by Theorem \ref{thm:tau}, we designed \newalgo\ such that the QC is satisfied with high probability. 
% However, since our QC is not monotone in quality, we cannot start eliminating agents until we explore all the arms $\tau$ times.  

% Next, we prove that after $\tau$ rounds, \newalgo\ will select a super-arm that satisfies the QC with high probability (w.h.p.)
\subsection{Ensuring Quality Constraints}
\label{sec:correct}
We provide Probably Approximate Correct (PAC) \cite{haussler1990probably,even2002pac} bounds on \newalgo\ satisfying QC after $\tau$ rounds:
\begin{theorem}
\label{thm:tau}
For $\tau=\frac{3 \ln T}{2\epsilon_2^2}$,  if each agent is explored  $\tau$ number of rounds, then if we invoke \dpss\ with target threshold $\alpha+\epsilon_2$ and  $\{(\hat{q}_i^t)^+\}_{i \in N}$ as the input, the QC is approximately met with high probability.
  \begin{displaymath}
    \mathcal{P}\left(q_{av}^t<\alpha-\epsilon_1\ | \ \frac{1}{s^t}\sum_{i \in S^t}(\hat{q}_i^t)^+\geq\alpha+\epsilon_2, t>\tau\right) \leq exp(-\epsilon_1^2t). 
  \end{displaymath}
%  where $\epsilon_1=\sum_{i \in S}(\tilde{q}_i^t ) -\alpha >0$
\end{theorem}
where $\epsilon_1$ is the tolerance parameter and refers to the planner's ability to tolerate a slighty lower average quality than required. 

Henceforth, a super-arm will be called \emph{correct} if it satisfies the QC approximately as described above.

\begin{proof} The proof is divided into two parts. Firstly, we show that for each $t > \tau$ round, the average value of $(\hat{q}_i^t)^+$ and that of $\hat{q}_i^t$ of the agents $i$ in selected super-arm $S^t$ is less than $\epsilon_2$. Secondly, we show that if the average of $\hat{q}_i^t$ is guaranteed to be above the threshold, then the average of $q_i$ over the selected agents would not be less than $\alpha - \epsilon_1$ with a high probability.

\begin{lemma}
\label{lemma:ucbmean}
The difference between the average of $(\hat{q}_i^t)^+$ and the average of $\hat{q}_i^t$ over the agents $i$ in $S^t$ is less than $\epsilon_2$, $\forall t > \tau$.
\end{lemma}
\begin{proof}
We have,
\begin{align*}
&\frac{1}{s^t}\sum_{i \in S^t} \left((\hat{q}_i^t)^+ - \hat{q}_i^t\right) = \frac{1}{s^t} \sum_{i \in S^t} \frac{\sqrt{3\ln{t}}}{\sqrt{2w_{i}^t}} \leq \frac{\sqrt{3\ln{t}}}{\sqrt{2w_{min}^t}}.
\end{align*}
where $w_{min}^t$ = $\min_{i}w_i^t$.
Since, for  $t < \tau$, we are exploring all the agents, thus, $w_i^\tau = \tau$. Now, since $w_i^t \geq w_i^\tau$, $\forall t>\tau$, thus, we claim that $w_{min}^t \geq \tau$ for $t>\tau$. Hence,  
\begin{align*}
\frac{\sqrt{3\ln{t}}}{\sqrt{2w_{min}^t}} \leq \frac{\sqrt{3\ln{T}}}{\sqrt{2\tau}}.
\end{align*}
For $\tau=\frac{3 \ln T}{2\epsilon_2^2}$, we have,
\begin{align*}
&\frac{1}{s^t}\sum_{i \in S^t} \left((\hat{q}_i^t)^+ - \hat{q}_i^t\right) \leq \epsilon_2.
\end{align*}
\end{proof}
\begin{lemma}
\label{lemma:meanavg}
$\forall t>\tau$
%If the sum of $\hat{q}_i^t$ for arms $i$ in the chosen super arm, $S^t$ is greater than or equal to $\alpha$, then the sum of $q_i$ for the same arms won't be less than $\alpha - \epsilon_1$ with a high probability. 
\begin{displaymath}
\mathcal{P}\left(\ q_{av}^t < \alpha - \epsilon_1 \ | \ \frac{1}{s^t}\big(\sum_{i \in S^t}\big(\hat{q}_i^t\ ) \geq \alpha \big)\ \right) \leq exp(-\epsilon_1^2t) .
\end{displaymath}
\end{lemma}{}

\begin{proof}
Let $Y^t = \frac{1}{s^t}\sum_{i \in S^t}\hat{q}_i^t$. Since E[$\hat{q}_i^t$] = E[$X_i^j$] = $q_i$, E[$Y^t$] = $q_{av}^t$. Hence, we have, 
% If arm $i$ is played $w_i^t$ number of times till round $t$,  arm as:
% \begin{align*}
% \hat{q}_i^t = \frac{1}{w_{i}^t}\sum_{j=1}^{w_i^t}X_i^j
% \end{align*}
% where $X_i^j$ is a Bernoulli Random Variable denoting the realized quality of the $j^{th}$ sample from arm $i$. Hence, E[$\hat{q}_i^t$] = E[$X_i^j$] = $q_i$.
% \begin{align*}
% \therefore q_{av}^t &= E\left[\frac{1}{s^t}\sum_{i \in S^t}\hat{q}_i^t\right]
% \end{align*}
% Now if, $Y^t = \frac{1}{s^t}\sum_{i \in S^t}\hat{q}_i^t$, we have, 
\begin{align*}
\mathcal{P}(E[Y^t] < \alpha - \epsilon_1) \ | \ Y^t \geq \alpha)
&\leq \mathcal{P}\big(Y^t \geq E[Y^t] + \epsilon_1)\\
&\leq exp(-\epsilon_1^2w^t).
\end{align*}
where $w^t = \sum_{i \in S^t} w_{i}^t$, i.e., total number of agents selected till round $t$. Since we pull atleast one agent in each round, we can say that, $w^t \geq t$. Thus, $\forall t>\tau$
\begin{align*}
\mathcal{P}\left(\ q_{av}^t < \alpha - \epsilon_1\ | \ \frac{1}{s^t}\left(\sum_{i \in S^t}\big(\hat{q}_i^t\ ) \geq \alpha \right)\ \right) \leq exp(-\epsilon_1^2t).
\end{align*}
\end{proof}{}
From Lemma \ref{lemma:ucbmean} and Lemma \ref{lemma:meanavg}, the proof follows.
\end{proof}

\subsection{Regret Analysis of \newalgo}
\label{sec:regret}
In this section, we propose the regret definition for our problem setting that encapsulates the QC. We then upper bound the regret incurred by \newalgo\ to be of the order $O(\ln T)$.

We define regret incurred by an algorithm $A$ on round $t$ as follows:
\begin{align*}
    Reg^t(A) = \begin{cases} 
      (r_{\boldsymbol{q}}(S^{\star}) - r_{\boldsymbol{q}}(S^t)) & \text{  if $S^t$ satisfies QC} \\
      L & \text{ otherwise}.
    \end{cases}
\end{align*}
where $S^{\star} = argmax_{S\in S_{f}} r_{\boldsymbol{q}}(S)$ and $L = \max_{S\in S_{f}} (r_{\boldsymbol{q}}(S^{\star}) - r_{\boldsymbol{q}}(S))$ is some constant. Here, $S_f$ are the feasible subsets which satisfies QC;  $S_{f} = \{S | S \subseteq N and \frac{\sum_{i \in s} x_i q_i}{\sum_{i \in s} x_i} \geq q_{av} \}$. \\
Hence, the cumulative regret in $T$ rounds incurred by the algorithm is:
\begin{equation}
    Reg(A) = \sum_{t=1}^T Reg^t(A).
\end{equation}

We now analyse the regret when the algorithm, $A$, is DPSS-UCB. 
\begin{align*}
    Reg(A) &= \sum_{t=1}^\tau Reg^t(A) + \sum_{t=\tau+1}^T Reg^t(A) \\
    &\leq  L\cdot \tau + \sum_{t=\tau+1}^T Reg^t(A) \\
    &\leq \frac{L \cdot 3 \ln T}{2\epsilon_2^2} + \sum_{t=\tau+1}^T Reg^t(A).
\end{align*}

Since our algorithm ensures that $S^t$ satisfies the approximate QC for $t > \tau$ with a probability greater than $1-\sigma$, where $\sigma = exp(-\epsilon_1^2t)$, we have, 
\begin{equation}
\label{eqn:regret}
    \mathbbm{E}[Reg(A)] \leq \frac{L \cdot 3 \log T}{2\epsilon_2^2} +  \left( \underbrace{ \sum_{t\geq \tau} \left[(1 - \sigma) (r_{\boldsymbol{q}}(S^{\star}) - r_{\boldsymbol{q}}(S^t))\right]}_{Reg_{u}(T)} +  \sigma L\right).
\end{equation}
where $S^t \in S_f$.

Now,
\begin{align*}
    \sum_{t\geq \tau} \sigma L &= \sum_{t\geq \tau} L e^{(-\epsilon_1^2t)} \leq \frac{L e^{(-\epsilon_1^2\tau)}}{1-e^{(-\epsilon_1^2)}}  \\
    & \sim O\left(\frac{1}{T^a}\right), \text{where } a = \frac{3\epsilon_1^2}{2\epsilon_2^2}.
\end{align*}

Now we bound the cumulative regret incurred after $t>\tau$ rounds when QC is satisfied, i.e., $Reg_{u}(T)$. Here we adapt the regret proof given by \cite{Chen13}. We highlight the similarities and differences of our setting with theirs and use it to bound $Reg_{u}(T)$.

\textbf{Bounding $Reg_{u}(T)$}:\\
\cite{Chen13} have proposed CUCB algorithm to tackle CMAB problem which they prove to have an upper bound regret of $O(\ln T)$. 
%To establish upper bound regret on $Reg_3(T)$, first we draw parallel between \newalgo\ and CUCB by highlighting the similarities and differences between the two settings.
Following is the CMAB problem setting considered in \cite{Chen13}:
\begin{itemize}
    \item There exists a constrained set of super-arms $\chi \subseteq 2^{N}$ available for selection.
    \item There exists an offline ($\eta, \nu$)-approximation oracle, ($\eta,\nu \leq 1$) s.t. for a given quality vector $\boldsymbol{q}'$ as input, it outputs a super-arm, S, such $\mathcal{P}(r_{\boldsymbol{q}'}(S) \geq \eta \cdot opt_{\boldsymbol{q}'}) \geq \nu$, where $opt_{\boldsymbol{q}'}$ is the optimal reward for quality vector $\boldsymbol{q}'$ as input.
    \item Their regret bounds hold for any reward function that follows the properties of monotonicity and bounded smoothness (defined below).
    \item Similar to our setting, they assume a semi-bandit feedback mechanism. 
%For a general reward function which satisfies two mild properties of monotonicity\footnote{This is different from monotonicity in QC defined in Jain \emph{et al.} } and bounded smoothness.

\end{itemize}{}
 %Then, we adapt their regret analysis to establish upper bound on $Reg_3(T)$ in \newalgo.
 
 Now, we state the reasons to adopt the regret analysis provided by \cite{Chen13} to bound $Reg_{u}(T)$
\begin{enumerate}
    \item We have shown that after $\tau$ rounds, we get the constrained set of super-arms, $\chi$, i.e., the set of super-arms that satisfies  QC), which forms a well defined constrained set, to select from in future rounds ($t > \tau$). 
    \item{We remark here that the utility function considered in our problem setting follows both the required properties, namely,}
  %  \begin{itemize}
  \\
     (i) \emph{Monotonicity:} The expected reward of playing any super-arm $S \in \chi$ is monotonically non-decreasing with respect to the quality vector, i.e., let $\boldsymbol{q}$ and $\overline{\boldsymbol{q}}$ be two quality vectors such that $\forall i \in N$, $q_i\leq \overline{q}_i$, we have $r_{\boldsymbol{q}}(S) \leq r_{\overline{\boldsymbol{q}}}(S)$ for all $S \in \chi$.
        Since our reward function is linear, it is trivial to note that it is monotone on qualities.\\
        (ii) \emph{Bounded Smoothness:} There exists a strictly increasing (and thus invertible) function $f(.)$, called bounded smoothness function, such that for any two quality vectors $\boldsymbol{q}$ and $\overline{\boldsymbol{q}}$, we have $r_{\boldsymbol{q}}(S) - r_{\overline{\boldsymbol{q}}}(S)$ $\leq$ f($\Lambda$) if $\max_{i \in S}$ $q_i$ - $\overline{q}_i$ $ \leq \Lambda$. As our reward function is linear in qualities, $f(\Lambda)= nR\times \Lambda$ is the bounded smoothness function for our setting, where $n$ is the number of agents.
         
 %   \end{itemize}{}
    \item \emph{Oracle}: Analogous to the oracle assumption in \cite{Chen13}, we have assumed the existence of an algorithm SSA (Section \ref{ssec:ssucb}). For DPSS-UCB, we use DPSS (Algorithm \ref{algo:dpss}) as our SSA . As DPSS provides exact solution, it acts as an $(\eta, \nu)$- approximate oracle for DPSS-UCB with $\eta = 1 = \nu$.
    %\citet{Chen13} assumes that for solving underlying combinatorial problem, there is a ($\eta, \nu$)-approximation oracle  ($\eta,\nu \leq 1$) s.t. for a given quality vector $\boldsymbol{q}$ as input, it outputs a super-arm such $\mathcal{P}(r_{\mu}(S) \geq \eta \cdot opt_{\boldsymbol{q}}) \geq \nu$, where $opt_{\boldsymbol{q}}$ is the optimal reward for quality vector $\boldsymbol{q}$ as input.\\
    %In our setting, we propose the algorithms namely \greedy\ and DPSS which acts as \oracle\ for our problem setting. From Theorem \ref{thm:greedy_opt}, we have, $\eta = \nu = 1$.
\end{enumerate}{}
 However, to ensure $\chi$ consists of all the correct super-arms, we need one additional property that should be satisfied, namely $\epsilon$-seperatedness property.
\begin{definition}
We say $\boldsymbol{q}=(q_1,q_2,\ldots,q_n)$ satisfies $\epsilon$-seperatedness if $\forall S \subseteq N$, $U(S) = \frac{1}{s}\sum_{i \in S}q_i$ s.t. $U(S) \not\in (\alpha-\epsilon, \alpha)$
\end{definition}
This suggests that there is no super-arm $S \in \chi$, such that $\alpha - \epsilon \leq \frac{1}{|S|}\sum_{i \in S}q_i^t \leq \alpha$. It is important for DPSS-UCB to satisfy $\epsilon_1$-seperatedness because if there exists such a super-arm, for which the average quality is between ($\alpha-\epsilon_1$, $\alpha$), \newalgo\ will include it in $\chi$ due to tolerance parameter $\epsilon_1$ while it would violate the QC.
 
\begin{theorem}
If qualities of the agents satisfy $\epsilon_1$-seperatedness, then $Reg_u(T)$ is bounded by $O(\ln T)$.
\label{thm:reg}
\end{theorem}
\begin{proof}
Following from the proof in \cite{Chen13}, we define some parameters.
A super-arm, $S$ is bad if $r_{\boldsymbol{q}}(S) < opt_{\boldsymbol{q}}$. Define $S_B$ as the set of bad super-arms. For a given underlying agent $i \in [n]$, define:
\begin{align*}
    \Delta_{\text{min}}^i = \text{opt}_{\boldsymbol{q}} - \text{max}\{r_{\boldsymbol{q}}(S) | S \in S_B, i \in S\}\\
    \Delta_{\text{max}}^i = \text{opt}_{\boldsymbol{q}} - \text{min}\{r_{\boldsymbol{q}}(S) | S \in S_B, i \in S\}.
\end{align*}{}
%Now, in the CUCB algorithm, they explore all arms for the first $n$ rounds, where $m$ is the number of agents. \ka{I don't think this statement is correct.} However, since we have already have explored this in the first $\tau$ rounds, we skip this exploration. Hence, 
Using the same proof as in \cite{Chen13}, we can show that, $V_T$,  the expected number of times we play a sub-optimal agent till round $T$, is upper bounded as:
\begin{align*}
    V_T &\leq n(l_T) + \sum_{t=\tau}^T\frac{2n}{t^2} \leq n(l_T) + \sum_{t=1}^T\frac{2n}{t^2}\\
    &\leq \frac{6n \cdot \ln T}{(f^{-1}(\Delta_{\text{min}}))^2} + \left(\frac{\pi^2}{3}\right) \cdot n.
\end{align*}{}
where $l_T = \frac{6 \ln T}{(f^{-1}(\Delta_{\text{min}}))^2}$. 
Hence, we can bound the regret as:-
\begin{align*}
    Reg_{u}(T) &\leq V_T\cdot \Delta_{\text{max}}\leq \left(\frac{6 \cdot \ln T}{(f^{-1}(\Delta_{\text{min}}))^2} + \frac{\pi^2}{3}\right)n \cdot \Delta_{\text{max}}\\
    &= \left(\frac{6 \cdot \ln T}{(\frac{\Delta_{\text{min}}}{R})^2} + \frac{\pi^2}{3}\right)n \cdot \Delta_{\text{max}}.
\end{align*}\end{proof}

Substituting the results of Theorem \ref{thm:reg} in Equation \ref{eqn:regret}, we prove that \newalgo\ incurs a regret of $O(\ln T)$.
\section{Greedy Approach}
\label{sec:greedyApp}
In the previous sections, we propose a framework and dynamic programming based algorithm to solve our subset selection problem for both when the agents' quality is known and not. Since \dpss\ explores all the possible combinations of the selection vector and the utility associated with it, the complexity of \dpss\ is of $O(2^n)$, which makes it difficult to scale when $n$ is large.

To overcome this limitation, we propose a greedy based approach to our problem. When the quality of agents are known, we propose \greedy\ that runs in polynomial time, $O(n\log n)$, and provides an approximate solution to our ILP. Then, we use \greedy\ as our SSA in the SS-UCB framework and propose GSS-UCB as an alternate algorithm to \newalgo\ in the setting where the qualities of the agents are unknown.

\subsection{Greedy Subset Selection (GSS)}
\label{sec:known}
Greedy algorithms have been proven effective to provide approximate solutions to ILP problems such as 0-1 knapsack. They do so by solving linearly relaxed variants of an ILP, such as fractional knapsack, and removing any fractional unit from its solution. We propose a similar algorithm for our subset selection problem by allowing $x_i \in [0,1]$. However, we cannot simply remove fractional units from our solution, as it may lead to QC violation. Consider the following example:

Given $n=2$ agents with qualities, $\boldsymbol{q} = [0.6,0.9]$, $\boldsymbol{c} = [10,100]$ and $\alpha = 0.7$. Allowing fractional units to be taken, the optimal solution would be to take $x_1 = 1, x_2 = 0.5$ units of the two agents. Removing fractional units would lead to selecting only the first agent, which violates the QC. Towards this, we include an additional step (Line \ref{greedy:s3}, Algorithm \ref{algo:gss}) in our algorithm that ensures that QC is not violated. Formally, the algorithm proceeds as follows:

\begin{enumerate}
    \item Divide the agents into the four categories, namely, $S_1$,$S_2$,$S_3$,$S_4$, as described in Section \ref{sec:dpss}.
    \item Select all agents in $S_1$. Let $d = \sum_{i \in S_1}(q_i - \alpha$) be the excess quality accumulated and as before, drop all agents in $S_4$.
    \item For agents in $S_2$, sort them in the decreasing order of revenue gained per unit loss in quality ($\frac{r_i}{\alpha-q_i}$). Similarly, for agents in $S_3$, sort them in the increasing order of revenue lost per unit gain in quality ($\frac{r_i}{\alpha-q_i}$).
    \item Select units (could be fractional) from agents from $S_2$ until the total loss of quality is no more than $d$. Essentially, we use the agents in $S_2$ to increase revenue while ensuring average quality is above the threshold.
    \item For agents in $S_2$ with remaining fractional units,
    we pair them up with an equivalent fractional unit of an agent in $S_3$ that balances the loss in average quality.
    \item When the revenue gained per unit loss in quality from the first non-exhausted agent in $S_2$ is less than the revenue lost per unit gain of quality from the first non-exhausted agent in $S_3$, terminate the algorithm. An agent is exhausted if the unit produce is completely selected. 
    \item For any agent in $S_3$ with a fractional unit, take the complete unit instead. For all other agents, remove any fractional units selected.
\end{enumerate}{}

%%%%%%%%%%%%%%%%%%%%%%%%%%%%%%%%%%%%%%
\begin{algorithm}[!ht]
\caption{\greedy}
\label{algo:gss}
%\algsetup{indent=0.75em}
\begin{small}
%\SetAlgoLined
% \KwResult{Write here the result }
\begin{algorithmic}[1]
\State \textbf{Inputs:} $N$, $\alpha$, $R$, costs  $\boldsymbol{c} = [c_i]$, qualities $\boldsymbol{q} = [q_i]$
\State \textbf{Output:} Quantities procured $\boldsymbol{x} = (x_1,\ldots,x_n)$
\State \textbf{Initialization:} $\forall i \in N$, $r_i = Rq_i - c_i$
\State Segregate $S_1$,$S_2$,$S_3$,$S_4$ as described in Section \ref{sec:dpss}
\State $\forall i \in S_1$, $x_i = 1$; $d = \sum_{i \in S_1}(q_i - \alpha)$
\State $\forall i \in S_4$, $x_i = 0$
\State $L_2 = sort(S_2)$ \text{on decreasing order of } $\frac{r_i}{\alpha-q_i}$
\State $L_3 = sort(S_3)$ \text{on increasing order of } $\frac{r_i}{\alpha-q_i}$
\State $p=0, q=0$
 \While{$d > 0$ and $p < |S_2|$} 
 \State $i = L_2[p]$; 
 \If{$\alpha - q_i \leq d$} $x_i = 1$, $d = d -\alpha - q_i$, $p += 1$
 \Else \ $x_i = \frac{d}{\alpha - q_i}$, $d=0$
 \EndIf  
 \EndWhile
 \While{$p < |S_2|$ and $q < |S_3|$}
 \State $i = L_2[p]$, $j = L_3[q]$
 \State $a = \frac{r_i}{\alpha-q_i}$, $b = \frac{r_j}{\alpha-q_j}$
 \If{$a \leq b$} \textbf{break};
 \EndIf  
 \State $w_1 = \min((1-x_i)(\alpha-q_i) , (1-x_j)(q_j-\alpha))$
 \State $x_i += \frac{w_1}{\alpha-q_i}$, $x_j += \frac{w_1}{q_j - \alpha}$
 \If{$x_i == 0$} $p++$;
 \EndIf  
 \If{$x_j == 0$} $q++$;
 \EndIf  
 \EndWhile
 \If{$0 < x_j < 1$} $x_j = 1$
 \EndIf  
 \label{greedy:s3}
 \State \textbf{return} $\lfloor\boldsymbol{x}\rfloor$
\end{algorithmic}
 \end{small}
\end{algorithm}
%%%%%%%%%%%%%%%%%%%%%%%%%%%%%%%%%%%%%%
%Format it properly

% Though the problem looks similar to the knapsack problem, it is not straightforward as the constraint is on the average quality of the agents and not on how many agents we can select. We design a greedy, deterministic, polynomial time (in the number of agents) algorithm, \emph{Greedy Subset Selection} (\greedy), to solve for the optimal subset selection when the qualities and costs %costs not discussed before
% are known to the center. 

% \begin{lemma}
% \greedy\ runs in $O(n \log n)$.
% \end{lemma}
% \begin{proof}
% It can be seen that steps 7 and 8 of the algorithm, run in $O(n \log n)$, while the two \emph{while} loops run in $O(n)$. Thus, \greedy\ runs in $O(n \log n)$. 
% \end{proof}

% \begin{theorem}
% \label{thm:greedy_opt}
% The output returned by GSS solves the centers subset selection problem formulated in ILP \ref{eq:opt_prb}.
% \end{theorem}{}
% \begin{proof}
% By design of the algorithm, we ensure that the quality is maintained above $\alpha$ at each step of the algorithm while trying to increase the utility. Due to space constraints, we skip a formal proof and can be provided if needed. 
% \end{proof}

\subsection{Approximation Ratio}
\label{ssec:approxRatio}
While \greedy\ is computationally more efficient than \dpss, it is important to note that it may not always return the optimal subset of agents. We show for the following example, that \greedy\ doesn't have a constant approximation w.r.t. the optimal solution:

Consider $n=3$ agents with qualities, $\boldsymbol{q} = [1.00,0.98,0.97]$ and $\boldsymbol{c} = [R-\epsilon,\frac{78R}{100}, \frac{47R}{100}]$. Hence, $\boldsymbol{r} = [\epsilon,\frac{R}{5}, \frac{R}{2}]$, where $R$ is some constant as discussed before such that $r_i = Rq_i - c_i$. If $\alpha = 0.99$, the value of $\frac{r_i}{\alpha - q_i}$ for the third agent is higher than that of the second, but only a fractional unit can pair with the first agent. Hence, according to \greedy, we only select the first agent giving us a utility of $\epsilon$, whereas the optimal utility is equal to $\epsilon+\frac{R}{5}$ corresponding to choosing the first and the third agent. Thus, the approximation ratio is $\frac{\epsilon}{\epsilon+\frac{R}{5}}$. Since $\epsilon$ can take an arbitrary small value, the approximation ratio between the utility achieved by \greedy\ and \dpss\ can be arbitrary small. 

However, through experiments, we show that in practice, \greedy\ gives close to optimal solutions at a huge computational benefit that allows us to scale our framework for a large number of agents, such as in an E-commerce setting.

\subsection{GSS-UCB}
\label{ssec:gssUCB}
When we use \greedy\ as the SSA in our SS-UCB framework, we refer to the algorithm as GSS-UCB. While the regret analysis may not necessarily hold, as \greedy\ does not have a constant approximation, we still show that in practice, it works as good as \newalgo\ in both (i) achieving constraint satisfaction after $\tau$ rounds and (ii) the regret incurred thereafter. We show this via experiments, as discussed in Section \ref{sec:exp}.

\section{Experimental Analysis}
\label{sec:exp}

\subsection{Subset Selection With Known Qualities}
In this section, we compare the performance of \greedy\ with \dpss\ in the setting where quality of the agents is known. In Figure \ref{fig:gssvdpss}, we compare the ratio of the utility achieved by \greedy\ ($z_{gss}$) to the utility achieved by \dpss\ ($z_{dpss}$) while ensuring the QC is met. In Figure \ref{fig:gssvdpss_all}, we present a box plot of the distribution of the ratios of these utilities over $1000$ iterations for $\alpha = 0.7$. To compare the performance of \greedy\ for much larger values of $n$, we compare it against the utility achieved by an ILP solver ($z_{ilp}$), namely, the COIN-OR Branch and Cut Solver (CBC) \cite{coin_ORcbc} since the computational limitations of \dpss\ made it infeasible to run experiments for large values of $n$. The results for the same are presented in Figure \ref{fig:gssvilp}. Lastly, in Table \ref{table:time}, we compare the ratio of the time taken by \greedy\ ($t_{gss}$) with respect to \dpss\ ($t_{dpss}$) and the ILP solver ($t_{ilp}$) for different values of $n$ with $\alpha$ being set to $0.7$. 

\subsubsection{Setup} For different values of $n$, the number of agents and $\alpha$, the quality threshold, we generate agents with $q_i$ and $c_i$ both $ \sim U[0,1]$. For Figure \ref{fig:gssvdpss} and \ref{fig:gssvilp}, we average our results over $1000$ iterations for each ($n$, $\alpha$) pair, while in Figure \ref{fig:gssvdpss_all}, we plot the distribution of the ratios obtained in each of the $1000$ iterations for different values of $n$ with $\alpha$ set to $0.7$. We use $R=1$ for all our experiments.

\begin{figure*}[ht]
\centering
\minipage[t]{0.64\linewidth}
\begin{subfigure}[t]{0.49\linewidth}
    % \label{fig:gssvdpss}
    \includegraphics[width=\linewidth]{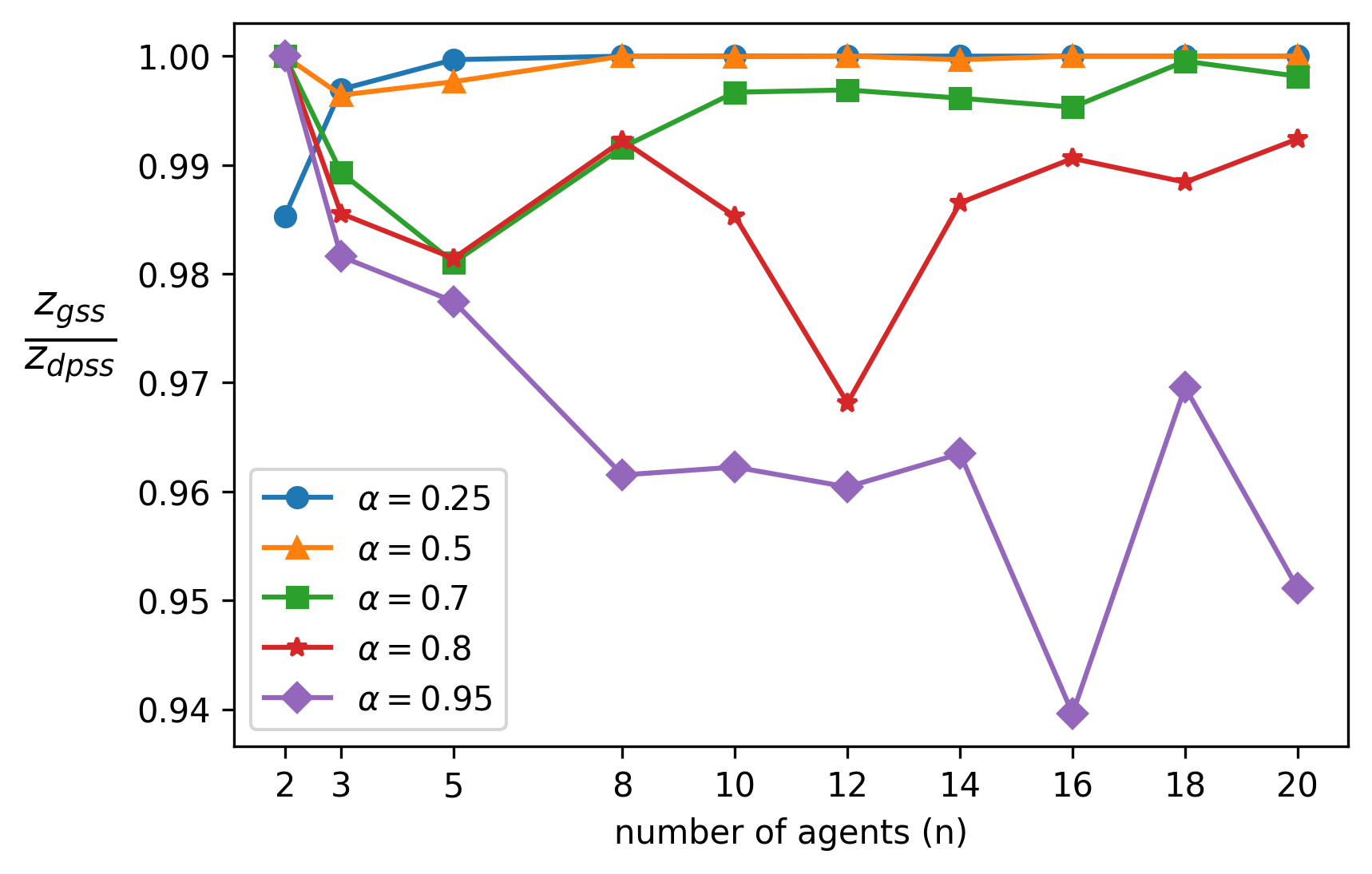}
    \caption{w.r.t DPSS }
    \label{fig:gssvdpss}
\end{subfigure}
% \label{fig:gssvdpss}
\begin{subfigure}[t]{0.49\linewidth}
    % \label{fig:gssvilp}
    \includegraphics[width=\linewidth]{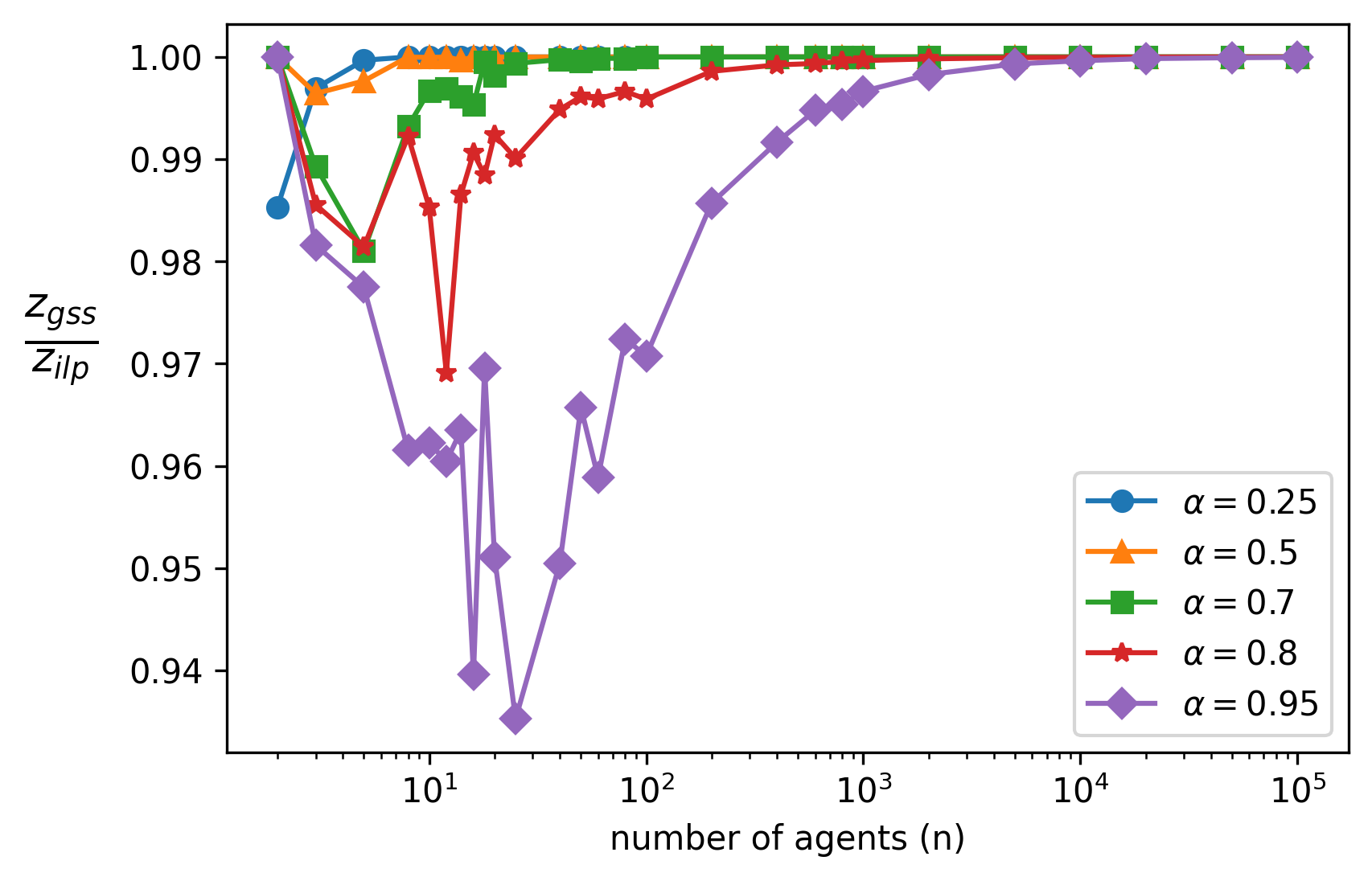}
    \caption{w.r.t. ILP }
    \label{fig:gssvilp}
\end{subfigure}
% \label{fig:gssvilp}
%Relative Performance of GSS to Optimal Solution
\caption{Performance of GSS on different values of $\alpha$}
\label{fig:gssperf}
\endminipage
~ % some horizontal spacing, can be done in different ways
\minipage[t]{0.32\linewidth}%
    \includegraphics[width=0.99\linewidth]{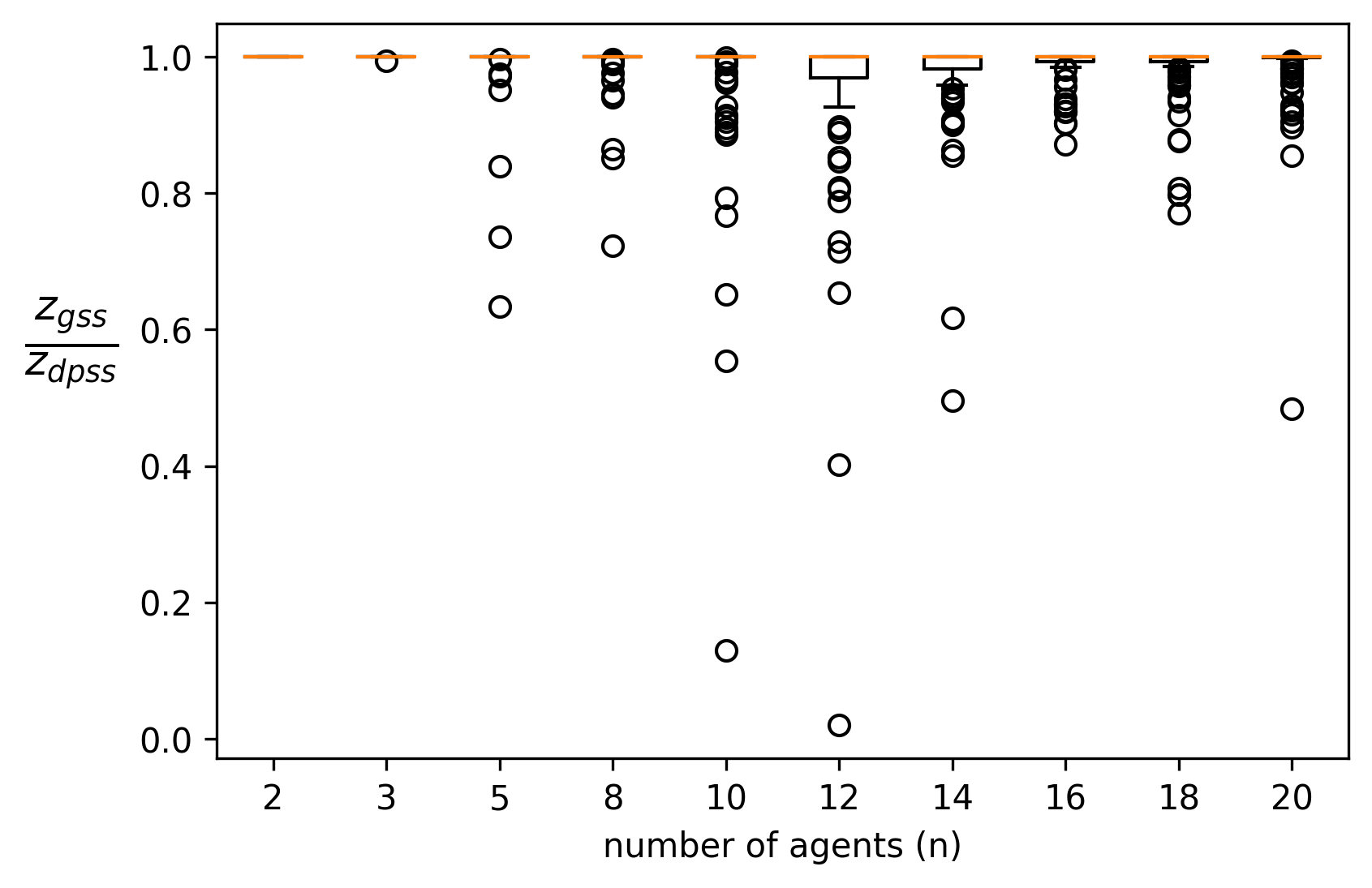}
    \caption{GSS vs DPSS ratio distribution}
    \label{fig:gssvdpss_all}
\endminipage
\end{figure*}

\subsubsection{Results and Discussion} 
As can be seen from Figures \ref{fig:gssvdpss} and \ref{fig:gssvilp}, the average ratio of both ($\frac{z_{gss}}{z_{dpss}}$) and ($\frac{z_{gss}}{z_{ilp}}$) lies approximately between [0.94,1.0], with a median of $1.0$ for almost all values of $n$ and only a few outliers and a few rare instances when the ratio drops below 0.2 as evident from Figure \ref{fig:gssvdpss_all}. This indicates that \greedy\ performs almost as good as \dpss\ in practice with an exponentially improving computational performance in terms of time complexity with respect to \dpss\ and an almost 50x improvement over the ILP solver as well. This establishes the efficacy of \greedy\ for practical use at scale.

\begin{table}[h]
\centering
\captionsetup{justification=centering}
 \begin{tabular}{||c | c | c||} 
 \hline
  n & $t_{dpss} : t_{gss}$ & $t_{ilp} : t_{gss}$ \\ [0.5ex] 
 \hline\hline
\textbf{2} &5.5 &70 \\ \hline
\textbf{5} &15.7 &64 \\ \hline
\textbf{8} &32.6 &63.7 \\ \hline
\textbf{10} &54.3 &58.6 \\ \hline
\textbf{12} &106.3 &67.6 \\ \hline
\textbf{14} &284.4 &65.3 \\ \hline
\textbf{16} &897.1 &60.2 \\ \hline
\textbf{18} &3109.7 &63.1 \\ \hline
\textbf{20} &11360.6 &68.1 \\ \hline
 \hline
\end{tabular}
% \caption{Describe table}
% \label{table1}
% \end{table}
\qquad
% \begin{table}[h]
% \centering
 \begin{tabular}{||c | c | c||} 
 \hline
 n & $t_{ilp} : t_{gss}$ \\ [0.5ex] 
 \hline\hline
 \textbf{25} &66.7 \\ \hline
\textbf{50} &58.3 \\ \hline
\textbf{100} &52.7 \\ \hline
\textbf{400} &43.1 \\ \hline
\textbf{1000} &31.8 \\ \hline
\textbf{5000} &31.6 \\ \hline
\textbf{10000} &34.5 \\ \hline
\textbf{50000} &45 \\ \hline
\textbf{100000} &56.8 \\ \hline
 \hline
\end{tabular}
\caption{Computational performance of GSS w.r.t. to DPSS and ILP}
\label{table:time}
\end{table}

\subsection{Subset Selection With Unknown Qualities}
In this section, we present experimental results of \newalgo\ and GSS-UCB towards the following:
\begin{enumerate}
    \item Constraint Satisfaction: As discussed in section \ref{sec:correct}, DPSS-UCB satisfies the QC approximately with high probability after $\tau = \frac{3\ln T}{2\epsilon_2^2}$ rounds. Here, $\alpha + \epsilon_2$ is the target constraint of the agent when $\alpha$ is the required average quality threshold. Towards this, we plot the average number of iterations where \newalgo\ and GSS-UCB returns a subset that satisfies QC at each round in our experiment for different values of $\epsilon_2$. 
    \item Regret incurred for $t > \tau$: We show that the regret incurred by our algorithm for $t>\tau$, follows a curve upper bounded by $O(\ln T)$. Towards this we plot the cumulative regret vs. round $t$, where $\tau < t \leq T$.
    %We compare the results for both \newalgo\ and GSS-UCB.
\end{enumerate}

\subsubsection{Setup} To carry out these experiments, we generated $n=10$ agents with both $q_i$, $c_i$ $ \sim U[0,1]$. We chose $\alpha = 0.7$ as for a higher value of $\alpha$ the number of super-arms satisfying QC is very low and hardly much to learn whereas for a low value, the number of super-arms that satisfy QC is very high but practically of not much interest. In Figure \ref{fig:correct}, we perform the experiment over a varied range of values of $\epsilon_2$, whereas in Figure \ref{fig:regret}, we set $\epsilon_2 = 0.01$. We average our results for 1000 iterations of each experiment. For example, in Figure \ref{fig:correct}, a value of $0.4$ at some round $t$, would denote that in 40\% of the iterations, the QC was satisfied at round $t$. For both the experiments, $R=1$ and $T=100000$.

\subsubsection{Discussion} Higher the value of $\epsilon_2$, higher is the target constraint and thus more conservative is our algorithm in selecting the subset of agents. Therefore, we achieve correctness quickly, which is evident from Figure \ref{fig:correct}. In all three cases, the algorithm achieves correctness in close to 100\% of the iterations, after $\frac{3\ln T}{2\epsilon_2^2}$ rounds (indicated by the vertical dotted line), which justifies our value of $\tau$. Similarly, the regret incurred by \newalgo\ for $t>\tau$ follows a curve upper bounded by $O(\ln T)$. The regret incurred by GPSS-UCB is slightly lower than \newalgo\ which further establishes the efficacy of our greedy approach.

\begin{figure}
\centering
\includegraphics[width=0.8\columnwidth, height=3.4cm]{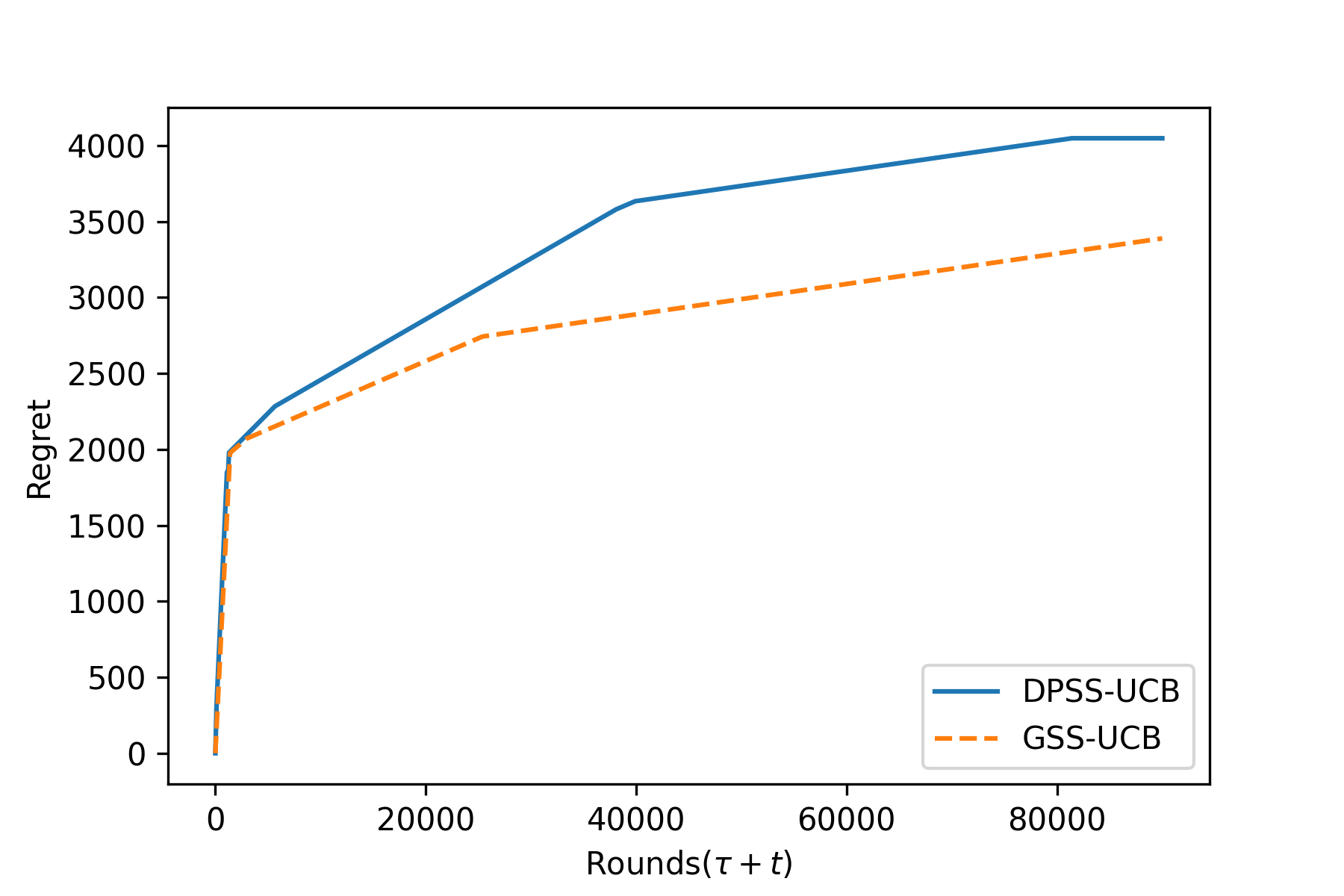}
\caption{Regret incurred for $t > \tau$}
\label{fig:regret}
\end{figure}

\begin{figure}
\centering
\includegraphics[width=0.8\columnwidth, height=3.4cm]{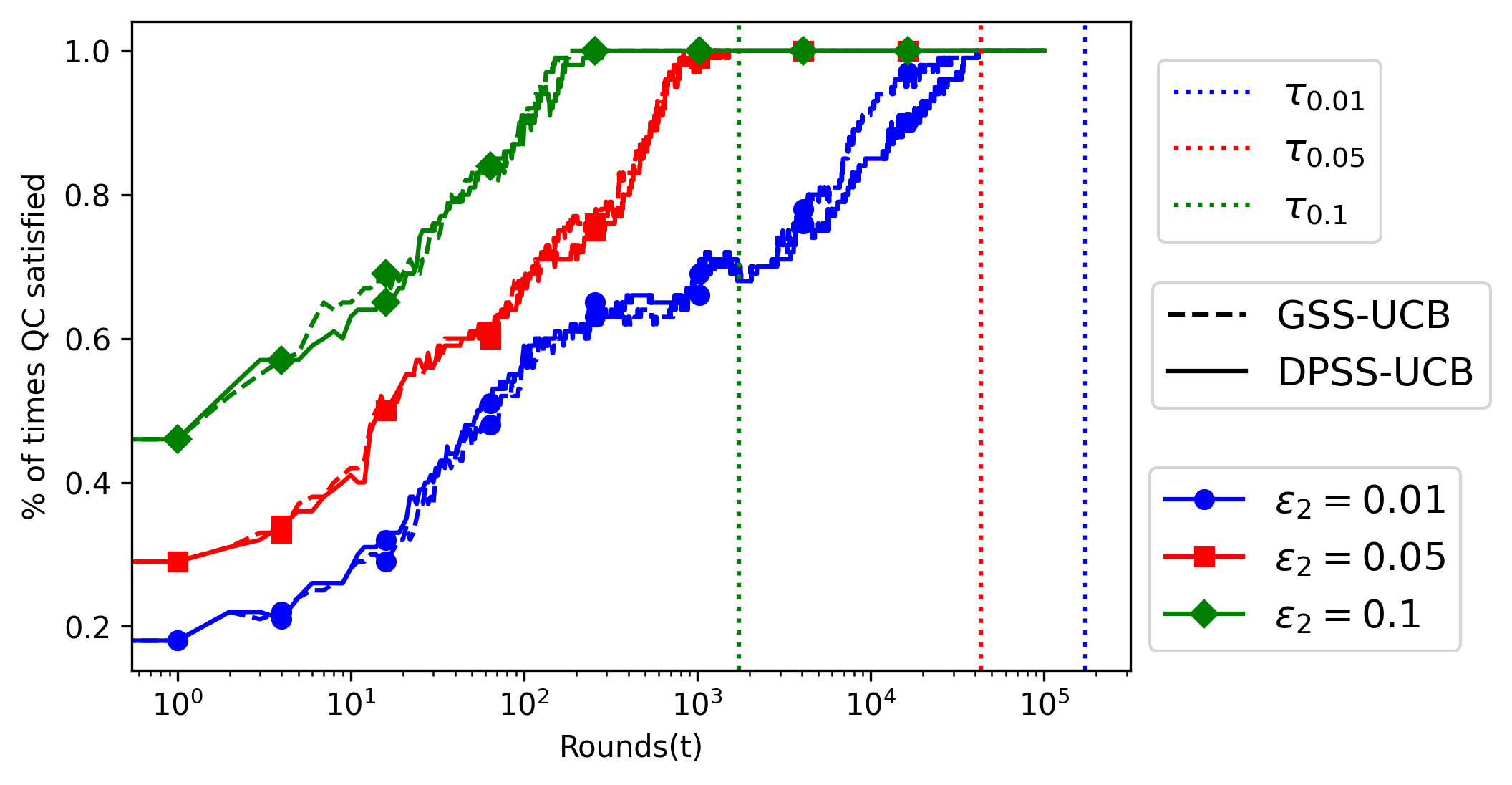}
\caption{Constraint satisfaction at each round}
\label{fig:correct}
\end{figure}

\section{Conclusion and Future Work}
\label{sec:concl}
In this paper, we addressed the class of problems where a central planner had to select a subset of agents that maximized its utility while ensuring a quality constraint. We first considered the setting where the agents' quality is known and proposed \dpss\ that provided an exact solution to our problem. When the qualities were unknown, we modeled our problem as a CMAB problem with semi-bandit feedback. We proposed SS-UCB as a framework to address this problem where both the constraint and the objective function depend on the unknown parameter, a setting not considered previously in the literature. Using \dpss\ as our SSA in SS-UCB, we proposed \newalgo\ that incurred a $O(\ln T)$ regret and achieved constraint satisfaction with high probability after $\tau = O(\ln T)$ rounds. To address the computational limitations of \dpss, we proposed  \greedy\ for our problem that allowed us to scale our framework to a large number of agents. Via simulations, we showed the efficacy of \greedy.

The SS-UCB framework proposed in this paper can be used to design and compare other approaches to this class of problems that find its applications in many fields. It can also easily be extended to solve for other interesting variants of the problem
such as (i) where the pool of agents to choose from is dynamic with new agents entering the setting, (ii) where an agent selected in a particular round is not available for the next few rounds (sleeping bandits) possibly due to lead time in procuring the units, a setting which is very common in operations research literature. Our work can also be extended to include strategic agents where the planner needs to design a mechanism to elicit the agents' cost of production truthfully.

\bibliographystyle{plain}
\bibliography{main}

\begin{thebibliography}{10}

\bibitem{achilleas2008marketing}
Kontogeorgos Achilleas and Semos Anastasios.
\newblock Marketing aspects of quality assurance systems: The organic food
  sector case.
\newblock {\em British Food Journal}, 110(8):829--839, 2008.

\bibitem{agrawal14}
Shipra Agrawal and Nikhil~R Devanur.
\newblock Bandits with concave rewards and convex knapsacks.
\newblock In {\em Proceedings of the fifteenth ACM conference on Economics and
  computation}, pages 989--1006, 2014.

\bibitem{auer2002}
Peter Auer, Nicolo Cesa-Bianchi, and Paul Fischer.
\newblock Finite-time analysis of the multiarmed bandit problem.
\newblock {\em Machine learning}, pages 235--256, 2002.

\bibitem{badanidiyuru13}
Ashwinkumar Badanidiyuru, Robert Kleinberg, and Aleksandrs Slivkins.
\newblock Bandits with knapsacks.
\newblock In {\em 2013 IEEE 54th Annual Symposium on Foundations of Computer
  Science}, pages 207--216. IEEE, 2013.

\bibitem{badanidiyuru14}
Ashwinkumar Badanidiyuru, John Langford, and Aleksandrs Slivkins.
\newblock Resourceful contextual bandits.
\newblock In {\em Conference on Learning Theory}, pages 1109--1134, 2014.

\bibitem{Biswas15}
Arpita Biswas, Shweta Jain, Debmalya Mandal, and Y.~Narahari.
\newblock A truthful budget feasible multi-armed bandit mechanism for
  crowdsourcing time critical tasks.
\newblock In {\em Proceedings of the 2015 International Conference on
  Autonomous Agents and Multiagent Systems}, AAMAS '15, pages 1101--1109,
  Richland, SC, 2015. International Foundation for Autonomous Agents and
  Multiagent Systems.

\bibitem{bubeck2012regret}
S{\'e}bastien Bubeck and Nicolo Cesa-Bianchi.
\newblock Regret analysis of stochastic and nonstochastic multi-armed bandit
  problems.
\newblock {\em Machine Learning}, 5(1):1--122, 2012.

\bibitem{shou_NIPS2014}
Shouyuan Chen, Tian Lin, Irwin King, Michael~R Lyu, and Wei Chen.
\newblock Combinatorial pure exploration of multi-armed bandits.
\newblock In {\em Advances in Neural Information Processing Systems 27}, pages
  379--387. 2014.

\bibitem{chenNIPS16}
Wei Chen, Wei Hu, Fu~Li, Jian Li, Yu~Liu, and Pinyan Lu.
\newblock Combinatorial multi-armed bandit with general reward functions.
\newblock In {\em Advances in Neural Information Processing Systems}, pages
  1659--1667, 2016.

\bibitem{Chen13}
Wei Chen, Yajun Wang, and Yang Yuan.
\newblock Combinatorial multi-armed bandit: General framework and applications.
\newblock In Sanjoy Dasgupta and David McAllester, editors, {\em Proceedings of
  the 30th International Conference on Machine Learning}, volume~28 of {\em
  Proceedings of Machine Learning Research}, pages 151--159, Atlanta, Georgia,
  USA, 17--19 Jun 2013. PMLR.

\bibitem{chen2016combinatorial}
Wei Chen, Yajun Wang, Yang Yuan, and Qinshi Wang.
\newblock Combinatorial multi-armed bandit and its extension to
  probabilistically triggered arms.
\newblock {\em The Journal of Machine Learning Research}, 17(1):1746--1778,
  2016.

\bibitem{combes15}
Richard Combes, Mohammad Sadegh Talebi~Mazraeh Shahi, Alexandre Proutiere,
  et~al.
\newblock Combinatorial bandits revisited.
\newblock In {\em Advances in Neural Information Processing Systems}, pages
  2116--2124, 2015.

\bibitem{combes_NIPS2015}
Richard Combes, Mohammad~Sadegh Talebi Mazraeh~Shahi, Alexandre Proutiere, and
  marc lelarge.
\newblock Combinatorial bandits revisited.
\newblock In {\em Advances in Neural Information Processing Systems 28}, pages
  2116--2124. 2015.

\bibitem{even2002pac}
Eyal Even-Dar, Shie Mannor, and Yishay Mansour.
\newblock Pac bounds for multi-armed bandit and markov decision processes.
\newblock In {\em International Conference on Computational Learning Theory},
  pages 255--270. Springer, 2002.

\bibitem{coin_ORcbc}
John~J. Forrest, Stefan Vigerske, Haroldo~Gambini Santos, Ted Ralphs, Lou
  Hafer, Bjarni Kristjansson, jpfasano, EdwinStraver, Miles Lubin, rlougee,
  jpgoncal1, h-i gassmann, and Matthew Saltzman.
\newblock coin-or/cbc: Version 2.10.5, March 2020.

\bibitem{gai10}
Yi~Gai, Bhaskar Krishnamachari, and Rahul Jain.
\newblock Learning multiuser channel allocations in cognitive radio networks: A
  combinatorial multi-armed bandit formulation.
\newblock In {\em IEEE Symposium on New Frontiers in Dynamic Spectrum}, pages
  1--9, 2010.

\bibitem{Gai12}
Yi~Gai, Bhaskar Krishnamachari, and Rahul Jain.
\newblock Combinatorial network optimization with unknown variables:
  Multi-armed bandits with linear rewards and individual observations.
\newblock {\em IEEE/ACM Transactions on Networking}, pages 1466--1478, 2012.

\bibitem{haussler1990probably}
David Haussler.
\newblock {\em Probably approximately correct learning}.
\newblock University of California, Santa Cruz, Computer Research Laboratory,
  1990.

\bibitem{ho13}
Chien-Ju Ho, Shahin Jabbari, and Jennifer~Wortman Vaughan.
\newblock Adaptive task assignment for crowdsourced classification.
\newblock In {\em International Conference on Machine Learning}, pages
  534--542, 2013.

\bibitem{hoeffding1963}
Wassily Hoeffding.
\newblock Probability inequalities for sums of bounded random variables.
\newblock {\em Journal of the American Statistical Association},
  58(301):13--30, 1963.

\bibitem{Jain2016}
Shweta Jain, Satyanath Bhat, Ganesh Ghalme, Divya Padmanabhan, and Y.~Narahari.
\newblock Mechanisms with learning for stochastic multi-armed bandit problems.
\newblock {\em Indian Journal of Pure and Applied Mathematics}, 47(2):229--272,
  Jun 2016.

\bibitem{jain18}
Shweta Jain, Sujit Gujar, Satyanath Bhat, Onno Zoeter, and Y~Narahari.
\newblock A quality assuring, cost optimal multi-armed bandit mechanism for
  expertsourcing.
\newblock {\em Artificial Intelligence}, 254:44--63, 2018.

\bibitem{karger2011}
David~R Karger, Sewoong Oh, and Devavrat Shah.
\newblock Iterative learning for reliable crowdsourcing systems.
\newblock In {\em Advances in neural information processing systems}, pages
  1953--1961, 2011.

\bibitem{kveton15}
Branislav Kveton, Zheng Wen, Azin Ashkan, and Csaba Szepesvari.
\newblock Tight regret bounds for stochastic combinatorial semi-bandits.
\newblock In {\em Artificial Intelligence and Statistics}, pages 535--543,
  2015.

\bibitem{papastavrou1996dynamic}
Jason~D Papastavrou, Srikanth Rajagopalan, and Anton~J Kleywegt.
\newblock The dynamic and stochastic knapsack problem with deadlines.
\newblock {\em Management Science}, 42(12):1706--1718, 1996.

\bibitem{rooderkerk2016robust}
Robert~P Rooderkerk and Harald~J van Heerde.
\newblock Robust optimization of the 0--1 knapsack problem: Balancing risk and
  return in assortment optimization.
\newblock {\em European Journal of Operational Research}, 250(3):842--854,
  2016.

\bibitem{sinha1979multiple}
Prabhakant Sinha and Andris~A Zoltners.
\newblock The multiple-choice knapsack problem.
\newblock {\em Operations Research}, 27(3):503--515, 1979.

\bibitem{slivkins2019}
Aleksandrs Slivkins.
\newblock Introduction to multi-armed bandits.
\newblock {\em Foundations and Trends{\textregistered} in Machine Learning},
  2019.

\bibitem{terziovski1997business}
Mil{\'e} Terziovski, Danny Samson, and Douglas Dow.
\newblock The business value of quality management systems certification.
  evidence from australia and new zealand.
\newblock {\em Journal of operations management}, 15(1):1--18, 1997.

\bibitem{thompson1933likelihood}
William~R Thompson.
\newblock On the likelihood that one unknown probability exceeds another in
  view of the evidence of two samples.
\newblock {\em Biometrika}, 25(3/4):285--294, 1933.

\bibitem{TRANTHANH2014}
Long Tran-Thanh, Sebastian Stein, Alex Rogers, and Nicholas~R Jennings.
\newblock Efficient crowdsourcing of unknown experts using bounded multi-armed
  bandits.
\newblock {\em Artificial Intelligence}, 214:89--111, 2014.

\bibitem{Tran13}
Long Tran-Thanh, Matteo Venanzi, Alex Rogers, and Nicholas~R Jennings.
\newblock Efficient budget allocation with accuracy guarantees for
  crowdsourcing classification tasks.
\newblock In {\em Proceedings of the 2013 international conference on
  Autonomous agents and multi-agent systems}, pages 901--908. International
  Foundation for Autonomous Agents and Multiagent Systems, 2013.

\bibitem{zaimai1989optimality}
GJ~Zaimai.
\newblock Optimality conditions and duality for constrained measurable subset
  selection problems with minmax objective functions.
\newblock {\em Optimization}, 20(4):377--395, 1989.

\end{thebibliography}

\end{document}